\documentclass[copyright,creativecommons]{eptcs}
\usepackage{breakurl}             
\usepackage{underscore}           

\usepackage{mathptmx}
\usepackage{comment}
\long\def\comment#1{}

\usepackage{algorithm}
\usepackage{algpseudocode}
\usepackage{amssymb,amsmath,verbatim,url}
\usepackage[english]{babel}

\usepackage{subfiles}
\usepackage[titletoc]{appendix}
\usepackage{amssymb,amsmath}
\usepackage{url}

\usepackage{color}
\usepackage{hyperref}
\usepackage{float}
\usepackage{newunicodechar}

\usepackage{graphicx}
\usepackage{enumitem}   
\usepackage{mdwlist}	
\usepackage{xspace}
\usepackage{xcolor}
\usepackage{csquotes}

\newcommand{\clingo}{\textsc{clingo} }

\newtheorem{theorem}{Theorem}[section]
\newtheorem{proposition}{Proposition}[section]

\newtheorem{definition}{Definition}[section]

\newunicodechar{⋀}{\bigwedge}
\newunicodechar{∈}{\in}
\newunicodechar{¬}{\neg}
\newunicodechar{∨}{\vee}
\newunicodechar{⋁}{\bigvee}
\newunicodechar{∧}{\wedge}
\newunicodechar{≠}{\ne}
\newunicodechar{⋂}{\bigcap}
\newunicodechar{←}{\leftarrow}
\newunicodechar{→}{\rightarrow}
\newunicodechar{∩}{\cap}
\newunicodechar{≥}{\ge}
\newunicodechar{∪}{\cup}
\newunicodechar{∖}{\setminus}
\newunicodechar{∞}{\infty}
\newunicodechar{…}{\ldots}

\usepackage{tikz}
\newenvironment{proof}{\paragraph{Proof:}}{\hfill$\square$}

\title{Mutex Graphs and Multicliques: Reducing Grounding Size for Planning}
\author{David Spies  \quad\qquad  Jia-Huai You  \quad\qquad Ryan Hayward
\institute{University of Alberta, Edmonton, Canada }
\email{\quad dnspies@gmail.com \quad\qquad jyou@ualberta.ca \quad\qquad  hayward@ualberta.ca}
}

\def\titlerunning{An Approximation Algorithm for Multiclique Covering}
\def\authorrunning{David Spies, Jia-Huai You \& Ryan Hayward}

\begin{document}
\maketitle
\begin{abstract}
We present an approach to representing large sets of {\em mutual exclusions}, 
also known as {\em mutexes or mutex constraints}. These are the types of constraints that specify the exclusion of some properties, events, processes, and so on.
They 
 are ubiquitous in many areas of applications. 
The size of these constraints for a given problem can be overwhelming enough to present a bottleneck for the solving efficiency of the underlying solver. In this paper, we propose a novel graph-theoretic technique based on {\em multicliques} for a compact representation of mutex constraints and apply it to 
domain-independent planning in ASP. As computing a minimum multiclique covering from a mutex graph is NP-hard, we propose an efficient approximation algorithm for multiclique covering and show experimentally that it generates substantially smaller grounding size for mutex constraints in ASP than the previously known work in SAT. 
\end{abstract}



\section{Introduction}

Mutual exclusion (mutex) can be traced back to concurrency control, which refers to the condition that prevents
simultaneous accesses to a shared resource. In knowledge representation, 
they specify the constraints 
that some properties cannot hold at the same time. 
For example, an object cannot be at different locations at the same time. These constraints frequently occur in applications from model-checking problems in computer-aided verification \cite{BiereCCZ99}, computer vision \cite{cvpr/MaL12,cvpr/RoyT14}, graph algorithms \cite{DBLP:books/daglib/0023091},
and AI planning \cite{rintanen2006compact}. 

The goal of this paper is to develop a graph-theoretic technique for compactly encoding large sets of mutex 
constraints and apply it to planning in ASP.  We do his by focusing on domain-independent AI planning as started out by SATPlan \cite{kautz2004satplan04}.  That is, we will first obtain an ASP planner 
by a straightforward translation from SATPlan and then study how to encode mutex constraints compactly for the planner. 


In SAT/ASP planning, mutex constraints are specified by formulas/rules that, for any state (which involves a time step, also called a {\em layer} in this paper), the actions with conflicting preconditions or 
effects, and 
the fluents that are inferred to be conflicting, are mutually exclusive. 
A naive encoding of these constraints
can certainly generate enough rules
to overwhelm the underlying solver for large planning instances. For example, in SAT planning these constraints can be expressed by 2-literal clauses (a 2-literal clause is of the form $l_1 \vee l_2$ where $l_1$ and $l_2$ are literals), which, according to \cite{rintanen2006compact},
constitute about 50-95\% of the formulas, and sometimes they used so much memory that they could not fit in a 32-bit address space. 

As shown in \cite{rintanen2006compact}, significant space-savings
can be gained by considering the way in which we encode mutex constraints.
We may
view the set of mutex constraints on fluents as an undirected graph, 
called a {\em mutex graph},
where each fluent is a vertex and each constraint is an edge.
When a
solver selects one fluent to be true at a given layer, it
infers by unit-propagation that each fluent joined directly by an edge
with the selection must be false. Thus, the set of fluents which are
true at a given layer constitutes {\em an independent set} on the mutex graph.\footnote{An independent set on a graph is a set of vertices where no two vertices
in the set share an edge \cite{robson1986algorithms}; equivalently this is a
clique in the complement graph.}

Rintanen \cite{rintanen2006compact}
shows that there exist other
smaller encodings besides the naive approach of listing out every
individual binary constraint and implies that since these encodings
are smaller, they must be superior. In their experiments, they use
instances of the AIRPORTS domain from an IPC planning competition.
This domain is notable because of the vast number of mutex constraints
it generates. The larger instances of this problem emit complex mutex
graphs which can overwhelm the underlying SAT solver if encoded naively
(in a one-constraint-per-edge fashion).

Rintanen further shows that the mutex graphs in these planning problems (even
in benchmark AIRPORTS) tend to be highly structured and that in SAT it is possible
to cover the mutex graph (somewhat more compactly) with cliques (complete
subgraphs) or with bicliques (complete bipartite subgraphs). A biclique
can be expressed in SAT using only one auxiliary variable and one binary
clause per assignment. Rintanen demonstrates that cliques can
be expressed using only a logarithmic set of bicliques. He concludes
that the best way to express a mutex graph in SAT is with a biclique
edge-covering.

In this paper, we show that for ASP, cardinality constraints
give us more power than is available in SAT and indeed we can directly
encode a mutex graph by its clique covering (without the extra cost
of a logarithmic factor), but further we can eliminate the choice
of whether to use cliques or bicliques entirely and instead cover 
the graph with {\em multicliques} (complete multi-partite subgraphs)
which is a generalization of both. Indeed, we find that with multicliques,
the number of clauses (namely ASP rules) and literals required to encode mutex constraints can be
further reduced over Rintanen's results. 

The next section provides an ASP planner as the context of dealing with mutex constraints. 
We also review the definitions of cliques/bicliques and comment on the complexity and representation issues. Section \ref{multi} then 
presents an approximation algorithm for multiclique covering and Section \ref{action-mutex} shows how to construct action mutex constraints simultaneously. In Section \ref{experiment} we present experimental results. 
Section \ref{related} comments on related work and Section \ref{conclusion} concludes the paper with final remarks. 

The ASP encodings in this paper are constructed to run on \clingo and follow the ASP-Core-2 Standard \cite{ASP-standard} except that (i) we will use ‘;’ to separate rule body atoms since the more conventional comma sign ‘,’ is overloaded and has a different meaning in more complex rules (\clingo supports both), and (ii) the disjunctive head of a rule may be written by a conditional literal.
The work reported here has been used in a recent construction of a cost-optimal planner in ASP \cite{spies2019}.

\section{Preliminaries}

\subsection{STRIPS Planning in ASP}
\label{ASPPlan}

We adopt a direct translation of 5 rules of
SATPlan \cite{kautz2004satplan04} into ASP and call the resulting planner ASPPlan.

\begin{verbatim}
rule 1.  holds(F,K) :- goal(F); finalStep(K).
rule 2.  happens(A,K-1) : add(A,F),validAct(A,K-1) :- holds(F,K); K > 0.
rule 3.  holds(F,K) :- pre(A,F); happens(A,K); validFluent(F,K).
rule 4.  :- mutexAct(A,B); happens(A,K); happens(B,K).
rule 5.  :- mutex(F,G); holds(F,K); holds(G,K).
\end{verbatim}
where $validAct(A,K)$ means that action $A$ can occur at time $K$ and $validFluent(F,K)$ means fluent $F$ can be true at time $K$.\footnote{Blum and Furst \cite{blum1997fast} give a handy way to identify for each
action and each fluent, what is the first layer at which this action/fluent
might occur by building the {\em planning graph}. Note that validAct/2 and validFluent/2 as well as predicates mutexAct/2 and mutex/2 are all extracted from the planning graph.}  Time steps used in constructing a plan are also called {\em layers}.

Rule 1 says that goals hold at the final layer.
In rule 2, if a fluent holds at layer $K$, the disjunction of actions that have
that fluent as an effect hold at layer $K-1$. The next rule says that actions at each layer imply their preconditions. The last two rules are mutex constraints: in rule 4, actions with (directly) conflicting preconditions or effects are
mutually exclusive, and in rule 5, the fluents that are inferred to be mutually exclusive are encoded as constraints.


Following SATPlan, we add to our plan ``preserving'' actions for each fluent.
The goal is to simulate the frame axioms by using the existing
machinery for having an action add a fluent that gets used some steps later. 
These preserving actions can be specified as:

\begin{verbatim}
action(preserve(F)) :- fluent(F).
pre(preserve(F),F) :- fluent(F).
add(preserve(F),F) :- fluent(F).
\end{verbatim}
where each fluent $F$ has a corresponding {\em preserving action} denoted by term $preserve(F)$. Preserving actions can be easily distinguished from regular actions. 
Now that an action occurs at time $K$ indicates that
its add-effect $F$ will hold at time $K+1$.

Note that the reason why rule 5 of ASPPlan prevents fluents from being deleted before they're
used is a bit subtle.
In order for a fluent to hold, it must occur in conjunction with a preserving
action at each time step it's held for. A preserving action has that fluent as a
precondition and so would be mutex with any action that has it as a delete
effect. This means that deleting actions cannot occur as long as that fluent
is held (by rule 4).\footnote{As a further note, when 
PDDL (planning domain definition language) without any extensions is defined, goals can only be positive and actions
can only have positive preconditions. There is a {\tt :negative-preconditions}
extension to PDDL, but we didn't use it.
Any problem which uses :negative-preconditions can be trivially
adapted to avoid using it by adding a fluent :not-F for every fluent :F and then
adding a corresponding add-effect wherever there's a delete-effect and vice
versa.}

Like SATPlan, we run this planner by solving at some initial makespan $K$, where $K$ is the first layer at which $validFluent(F,K)$ holds for all $ goal(F)$,
and if it is UNSAT, we increment $finalStep$ by 1 until we find a plan.

This is a straightforward and unsurprising encoding in every respect,
but has a somewhat surprising consequence as compared to SATPlan. Because ASP models are
stable, for any fluent $F$, $holds(F,K)$ can
only be true if there exists some action which requires its truth
as per rule 3. Similarly for actions as per rule 2.
Furthermore, since rule 2 is disjunctive at every step, the set of actions which occurs is a minimal
set required to support the fluents at the subsequent step. This conforms
exactly to the approach to planning by Blum and Furst  \cite{blum1997fast}: First build
the planning graph, then start from the goal-state planning backwards,
at each step selecting a minimal set of actions necessary to add all
the preconditions for the current set of actions. That is, in this ASP translation, the neededness-analysis as carried out in \cite{robinson2008compact} is accomplished automatically during  grounding
or
during the search for stable models. 


\medskip
{\bf Smart Encoding of Action Mutexes:}
Let us first consider action mutex constraints as expressed by 
Rule 4 of ASPPlan, which 
can blow up in size when grounded because nearly any two actions 
acting on the same
fluent can be considered directly conflicting.
For example, assume a planning problem in which there is a crane
which we must use to load boxes onto freighters and there are many
boxes and many freighters available but only one crane. Then we will
have one such constraint for every two actions of the form,
$load(Crate,Freighter)$,
 for any crate and any freighter. As there is already a quadratic number of actions in the problem description size ($crates \times freighters$),
the number of
mutex constraints over \emph{pairs} of actions is \emph{quartic} in the initial
(non-ground) problem description size.

We would like to avoid such an explosion by introducing new predicates
to keep the problem size down. We will only consider two actions to
be mutex if one deletes the other's precondition. But we will take
extra steps to ensure that no add-effect is later used if the same
fluent is also deleted at that step. Here is the revised encoding of
rule 4.

\begin{verbatim}
used_preserved(F,K) :- happens(A,K); pre(A,F); not del(A,F).
deleted_unused(F,K) :- happens(A,K); del(A,F); not pre(A,F).
:- {used_preserved(F,K); deleted_unused(F,K);
    happens(A,K) : pre(A,F), del(A,F)} > 1; valid_at(F,K).

deleted(F,K) :- happens(A,K); del(A,F).
:- holds(F,K); deleted(F,K-1).
\end{verbatim}

Effectively, we are splitting the ways in which we care that an action
$A$ can relate to a fluent $ F$ into three different cases: 
(i) $A$ has $F$ as a precondition, but not a delete-effect;
(ii) $A$ has $F$ as a delete-effect, but not a precondition; and 
(iii) $A$ has $F$ as both a precondition and a delete-effect.



By explicitly creating two new predicates for properties (i) and (ii),
we have packed this restriction into one big cardinality constraint.
Further, we must account for conflicting effects, so we define one
more predicate ({\tt deleted/2}) which encapsulates the union of all actions
from properties 2 and 3 (those that delete $F$) and assert that $F$ cannot
hold at this step if any of those actions occurred in the previous
one.\footnote{There is a minor difference between the definition of mutex as given in \cite{blum1997fast}, which appears to be overly restrictive, and the definition we're using. Whereas graphplan treats any two actions as mutex if they have conflicting effects (one adds a fluent which the other deletes), we only consider them to be mutex if they have conflicting effects and the add-effect is used at that layer. So we allow actions to occur simultaneously with conflicting effects as long as the relevant fluent doesn't hold afterwards.}

\subsection{Cliques and Bicliques}
\label{clique}

We review the definitions of cliques and bicliques and comment on their possible encodings in SAT. 

Let $G = (V,E)$ be an undirected graph. A {\em clique} is a subgraph $(C,E')$ of $G$ such that $C \subseteq V$ and $E' = \{(v,u) \in E \,|\, v, u \in C, u \not = v\}$.  A {\em biclique} is a subgraph $(C, C', E')$ of $G$ such that $C, C' \subseteq V$, $C \cap C' = \emptyset$, and $E' = \{(u,v) \in E \,|\, u \in C, v \in C'\}$.

That is, cliques are complete subgraphs of a graph and 
 bicliques are complete bipartite subgraphs of a graph. Deciding if a graph has a clique of size $n$ is known to be NP-complete \cite{GareyJ79,DBLP:conf/coco/Karp72}.  
This is also the case for bicliques under several size measures \cite{GareyJ79,Peeters03,Yannakakis78}. There are approximation algorithms for the computation of cliques and bicliques, with approximation guarantees 
\cite{Hochbaum98}, or without \cite{rintanen2006compact}.

In SAT, given $n$ fluents, besides the naive $O(n^2)$ size representation, cliques can be represented in size $O(n)$ using $O(n)$ many auxiliary variables, or in size $O(n \log n)$ using only $O(\log n)$ many auxiliary variables. Bicliques enjoy a more compact representation: if $C$ and $C'$ form a biclique, then $|C| \times |C'|$ many
binary constraints can  be represented by $|C| + |C'|$ many 2-literal clauses using only one auxiliary variable \cite{rintanen2006compact}. The idea is that for any literals $l \in C$ and $l' \in C'$, mutex constraints of the form $l \vee l'$ can all be represented using one new variable, say  $x$, by $\neg l \rightarrow x$ and $x \rightarrow l'$.

\section{An Approximation Algorithm for Multiclique Covering}
\label{multi}

In this section, we formulate a polynomial-time, approximation algorithm for multiclique covering. 
First, let us have a formal definition of multiclique.
\begin{definition}
Let $G = (V,E)$ be an undirected graph. 
A {\em multiclique} of $G$ is a subgraph $(C_1, \dots, C_k, E')$ of $G$, 
such that $C_1 \cup \dots \cup C_k \subseteq V$, $ C_i \cap C_j =\emptyset$ for all $1 \leq i,j\leq k$ where $i \not = j$, and 
$E' = \{(u,v) \in E \,|\, u \in C_i, v \in C_j, i \not = j\}$.

We call each $C_i~(1 \leq i \leq k)$ above a {\em partition}. 
\end{definition}


\begin{proposition}
A multiclique is a graph whose complement is a cluster graph, i.e., a set of disjoint cliques.
\end{proposition}

The claim is easy to verify. 
Consider any graph $G$ which is a multiclique by definition.
In the complement graph $G^{c}$, every partition is a clique. Further,
since any two vertices $u$ and $v$ must have an edge if they belong
to separate partitions in $G$, it follows that there are no edges
between partitions in $G^{c}$, therefore, the only edges in $G^{c}$
belong to cliques.
Similarly, if $G^{c}$ is a cluster-graph, then the connected components
form the partitions in $G$ as a multiclique.

Given a mutex graph, a naive encoding of mutex constraints in ASP is to list each edge between two vertices by a 2-literal constraint.  With a multiclique covering, mutex constraints in a mutex graph can be encoded compactly. 

Given a graph $G = (V,E)$, the goal of {\em multiclique covering} is to produce a sequence of multicliques $\Pi = (S_1, \dots, S_n)$ for some $n$, where each $S_i$ is a multiclique subgraph of $G$, for all $j > i$, $S_j$ contains at least one edge not in $S_i$, and the union of edges in $S_k~(1\leq k \leq n)$ is $E$.  In the multiclique covering $\Pi$, $S_i$ and $S_j$ may share some vertices. In general, to cover all mutex constraints in a mutex graph, the edges covered in different multicliques in $\Pi$ need not be non-overlapping. In our algorithm, we do allow overlapping if it leads to more compact representation. In summary, as the edges in a mutex graph represent constraints, multiclique covering is to cover the edges of the mutex graph where the edges are spread out in multicliques that are constructed.

For each multiclique constructed, 
we can encode a constraint graph in ASP
as:

\begin{verbatim}
% Covering is given by inPartition(F,P) if fluent F belongs to partition P, 
% and inMulticlique(P,M) if partition P belongs to multiclique M.
% p(P,K): P is a partition at layer K.
partitionHolds(P,K) :- holds(F,K); inPartition(F,P).
:- {p(P,K): partitionHolds(P,K),inMulticlique(P,M)} > 1;
   multiclique(M); layer(K).
\end{verbatim}

Here we have a cardinality constraint expressing the rule that among
all partitions $P$ of multiclique $M$, at most one holds at layer
$K$. Furthermore, if any fluent $F$ holds at layer $K$, so does its
partition $P$.

Additionally, we can avoid some unnecessary rules by handling
singleton partitions
specially. A singleton partition can be packed \emph{directly} into the cardinality
constraint rather than introduced through an auxiliary atom:
\begin{verbatim}
:- {partitionHolds(P,K):inMulticlique(P,M);
      holds(F,K):singletonPartitionOf(F,M)} > 1;
    multiclique(M); layer(K).
\end{verbatim}


Now, our ASPPlan given in Section \ref{ASPPlan} is updated by replacing Rule 5 therein with the above rules. 

Hence, 
given a planning instance, if we can construct a
multiclique covering from its mutex graph, we can use ASP to encode these constraints compactly.
Now let us find an algorithm for this task.

In general, finding a minimum multiclique covering (using as few multicliques as
possible) is NP-hard.
To see this, consider the problem of finding a minimum multiclique
covering on a bipartite graph. It's easy to see that a multiclique on a
bipartite graph is a biclique. Thus the minimum multiclique covering of a
bipartite graph is the minimum biclique covering. The size of the minimum
biclique covering of a bipartite graph is also known as its
\emph{bipartite dimension}. Finding the bipartite dimension
of a graph is known to be NP-hard \cite{amilhastre1997computing}.
Thus, finding a minimum
multiclique cover is also NP-hard. 

Nonetheless, we can still use approximation algorithms similar to
those used in \cite{Hochbaum98}. One critical observation
is that under the restriction that a multiclique must use
exactly a particular set of vertices, there is always only one optimal
way to partition those vertices into a multiclique to cover a maximal
set of edges:
If there is a path between two vertices $v$ and $w$ in the complement
of the induced graph, then they must belong to the same partition.
If there is no path, then we might as well put them in separate partitions.
Therefore, the best partition is the one which makes a partition for
each connected component in the complement of the induced graph.

\begin{algorithm}[t]
\label{covering0}
\caption{Multiclique Covering}
\begin{algorithmic}[1]
\Procedure{find\_cover}{$g :: \mbox{Graph}$} $\rightarrow$ Set MultiClique
  \State $\mbox{var}\ uncovered \gets g.edges :: \mbox{Set Edge}$
  \State $\mbox{var}\ multicliques \leftarrow  {\{\} :: \mbox{Set MultiClique}}$

  \While{$uncovered.nonempty$} 
      \State   $new\_multiclique \leftarrow \textsc{next\_multiclique}()$ 
      \State  $multicliques \leftarrow multicliques \cup \{new\_multiclique\}$ 
      \State  $uncovered\leftarrow uncovered \setminus \textsc{edges\_covered\_by}(new\_multiclique)$
  \EndWhile
  \State \Return multicliques
\EndProcedure
\end{algorithmic}
\end{algorithm}

Let us use an example to illustrate.
Consider the mutex graph $G= (V,E)$ on the left of the figure below and its complement graph $G^c$ on the right. The connected components of $G^c$ give us a multiclique $\{\{a,b,d\}, \{c\}, \{e\}\}$, which covers almost all edges in $E$ except edge $(a,b)$. So edge $(a,b) \in E$ will have to be captured in another multiclique. 

\vspace{-0.2in}
\begin{figure} [h!]
\label{graph0}
\begin{center}
\begin{tikzpicture}

\node[shape=circle,draw=black,scale=0.8] (v1) at (-8,1) {$a$};
\node[shape=circle,draw=black,scale=0.8] (v2) at (-6,1) {$e$};
\node[shape=circle,draw=black,scale=0.71] (v3) at (-9,0) {$b$};
\node[shape=circle,draw=black,scale=0.8] (v4) at (-7,0) {$c$};
\node[shape=circle,draw=black,scale=0.69] (v5) at (-5,0) {$d$};

\node[shape=circle,draw=black,scale=0.8] (v6) at (-2,1) {$a$};
\node[shape=circle,draw=black,scale=0.8] (v7) at (0,1) {$e$};
\node[shape=circle,draw=black,scale=0.71] (v8) at (-3,0) {$b$};
\node[shape=circle,draw=black,scale=0.8] (v9) at (-1,0) {$c$};
\node[shape=circle,draw=black,scale=0.69] (v10) at (1,0) {$d$};

\draw (v6) edge node{} (v10);
\draw (v8) edge [bend right] node{} (v10);
\draw (v1) edge node{} (v2);
\draw (v1) edge node{} (v3);
\draw (v1) edge node{} (v4);
\draw (v2) edge node{} (v4);
\draw (v2) edge node{} (v5);
\draw (v3) edge node{} (v4);
\draw (v4) edge node{} (v5);
\draw (v2) edge [bend right=75] node{} (v3);
\end{tikzpicture}
\begin{quote}
\vspace{-.1in}
\caption{An example mutex graph and its complement.}
\end{quote}
\end{center}
\end{figure}
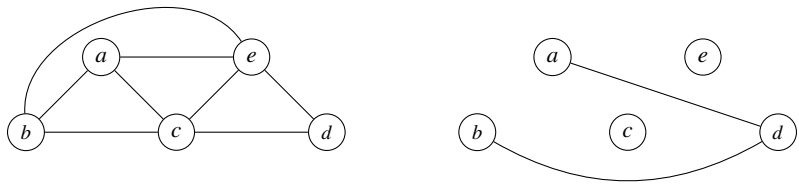

\vspace{-0.3in}

Our algorithm is given in Algorithm 1, with supporting functions given in Algorithm 2. The algorithm
is greedy, simple, and polynomial-time. We track the set of uncovered
edges and tack multicliques on one at a time, greedily building each multiclique
in such a way so as to maximize the difference $\phi_1 - \phi_2$, where
$\phi_1$ is the number of literals in the naive encoding and $\phi_2$ is the number of literals in our ASP encoding of the corresponding multiclique. 
A difference indicates an encoding reduction.

\begin{algorithm}[t!]
\label{alg:multiclique-cover}
\caption{Multiclique Covering Helper Functions}
\begin{algorithmic}[1]

\State type MCPartition = Set Vertex
\State type MultiClique = Set MCPartition

\Function{make\_multiclique}{$vs :: \mbox{Set Vertex}$} $\rightarrow$  MultiClique
  \State \Return $g.induced\_subgraph(vs).complement().connected\_components()$
\EndFunction

\Function{edges\_covered\_by}{$mc :: \mbox{MultiClique}$} $\rightarrow$ Edge
  \State \Return $\{(x,y) \mid  p \in mc,  q \in mc,  p\not = q,  x \in p,  y \in q\}$
\EndFunction

\Function{count\_uncovered\_incident\_edges}{$x :: \mbox{Vertex}$} $\rightarrow$ $\mathbb{N}$
  \State $|(g.incident\_edges(x) \cap uncovered)|$
\EndFunction

\Procedure{defaults\_for}{$vs :: \mbox{Set Vertex}$}$\rightarrow$  MCPartition
  \State $candidates\leftarrow  \bigcap \{g.neighbors(v) \mid v \in vs\}$ :: Set Vertex
  \State \Return $\{c \mid c \in candidates, ~|g.incident\_edges(c) \cap  uncovered| \geq 2\}$
\EndProcedure

\Procedure{score}{$vs :: \mbox{Set Vertex}$} $\rightarrow$ $\mathbb{Z}$
  \State $multiclique$ :: MultiClique
  \State $multiclique \leftarrow \textsc{make\_multiclique}(vs) \cup {\textsc{defaults\_for}(vs)}$
  \State $newly\_covered :: \mbox{Set Edge}$
  \State $newly\_covered\leftarrow \textsc{edges\_covered\_by}(multiclique) \cap uncovered$
  \State $complexity\_cost$ :: $\mathbb{Z}$
  \State $complexity\_cost\leftarrow  \displaystyle\sum_{p ∈ multiclique}{\begin{cases}
    1 & \mbox{if } |p| = 1\\
    2*|p|+1 & \mbox{if } |p| > 1
  \end{cases}}$
  \State \Return $2 * |newly\_covered| - complexity\_cost$
\EndProcedure

\Procedure{next\_multiclique}{} $\rightarrow$  MultiClique
  \State $first\_vertex :: \mbox{Vertex}$
  \State $first\_vertex\leftarrow {\arg\max}_{g.vertices}{(\lambda w . \ \textsc{count\_uncovered\_incident\_edges}(w))}$
  \State $\mbox{var}\ vertex\_set\leftarrow \{first\_vertex\} :: \mbox{Set Vertex}$
  \Repeat
    \State $next :: \mbox{Vertex}$
    \State $next\leftarrow {\arg\max}_{g.vertices}{(\lambda w.\ \textsc{score}(vertex\_set \cup \{w\}))}$
    \State $improved\leftarrow \textsc{score}(vertex\_set \cup \{next\}) > \textsc{score}(vertex\_set)$
    \If{$improved$}
      \State $vertex\_set\leftarrow vertex\_set \cup \{next\}$
    \EndIf
  \Until{$¬improved$}
  \State \Return {$\textsc{make\_multiclique}(vertex\_set \cup \textsc{defaults\_for}(vertex\_set))$}
\EndProcedure
\end{algorithmic}
\end{algorithm}

For more details, in Algorithm 1, the variable {\em multicliques} is empty to start with.
Then it iteratively adds one new multiclique at a time until all edges are covered.

\newpage
In the helper function \textsc{next\_multiclique} in Algorithm 2,
we select the first vertex by finding the one incident to the most
uncovered edges.  This is accomplished at Line 27 (line numbers below all refer to Algorithm 2), where we use a lambda function which is applied to each vertex for the parameter $w$.
We then repeatedly select each subsequent vertex to greedily
maximize the size difference mentioned above under the assumption that we will finish by adding
on a ``default partition'' of vertices, until no improvement can be generated
(lines 29-36 of Algorithm 2). The default partition consists
of all vertices which have an edge to every vertex we have  selected
so far including at least two edges not yet covered 
(lines 12-15).\footnote{If there is only one, there will be no savings in encoding size, as it would require the same number of literals/rules to include a vertex in a partition.}


Given a set of vertices $vs$, the function \textsc{make\_multiclique}$(vs)$ generates a multiclique, where the partitions are obtained by finding the connected components of the complement graph induced from $vs$, along with the covered edges (lines 3-5). 

Note that, instead of \emph{removing} edges from the graph once they've been assigned to a multiclique, we keep a separate record of ``uncovered'' edges which still remain to be assigned. In this way the same edge may be covered twice by different multicliques if that helps minimize the encoding (cf. line 27).

\begin{theorem}
Given a mutex graph $G=(V,E)$, the algorithm \textsc{find\_cover} terminates after a number of execution steps in polynomial time in the size of $G$, and after termination, a sequence of multicliques 
$\{(V_1,E_1), \dots, (V_n, E_n)\}$ is generated such that $V_1 \cup \dots \cup V_n \subseteq  V$ and 
 $E_1 \cup \dots \cup E_n = E$.
\end{theorem}
\begin{proof}
First, we verify that each $(V_i, E_i)~(1 \leq i \leq n)$ is a multiclique. $V_i$ is returned as a set of vertices by \textsc{next\_multiclique} and partitioned by \textsc{make\_multiclique} into $C_1, \dots C_k$ satisfying the following statement: for any $i \not = j$, $v \in C_i$
and $v' \in C_j$ iff there is no path between $v$ and $v'$ in the complement of the graph induced from $V_i$ iff there is an edge between $v$ and $v'$ in the given mutex graph. Hence, each vertex in any partition is connected to every vertex in a different partition. Then, to obtain a multiclique, we only need to let $E_i$ be 
the set of edges that connect vertices of different partitions. 

The algorithm terminates since each $E_j$ covers at least one of the uncovered edges. The first vertex is selected such that it maximizes the number of uncovered edges to which it's incident, so as long as there are uncovered edges, we're guaranteed to select a first vertex which is incident to at least one of them. Trivially we can extend this to a multiclique which covers an uncovered edge by selecting the vertex on the other side of any one of them for a score of at least zero. Since the score of the multiclique is only allowed to improve from there and the score measures the number of uncovered edges we've covered, it must be the case that every multiclique will cover at least one new uncovered edge (otherwise its score would be negative).


The claim on polynomial time holds because the number of multicliques is bounded by $|E|$ and 
there are at most $|E|$ calls to \textsc{next\_multiclique}; further, it can be  easily checked that the computation of each such call takes polynomial time. 
\end{proof}

\medskip \medskip
Let's take a look at how this behaves on an example graph. We'll start with a
mutex graph for a ferry crossing problem in which we have three islands, a ferry
and a car. The ferry can be at any of the three islands and it can have just
moved or be in the process of loading. The car can be on the ferry or at one of
the three islands. If loading then the car is not currently on the ferry.
Figure \ref{ferry1} shows what the mutex graph for the problem looks like.
\begin{figure}[t]
\includegraphics[scale=0.5]{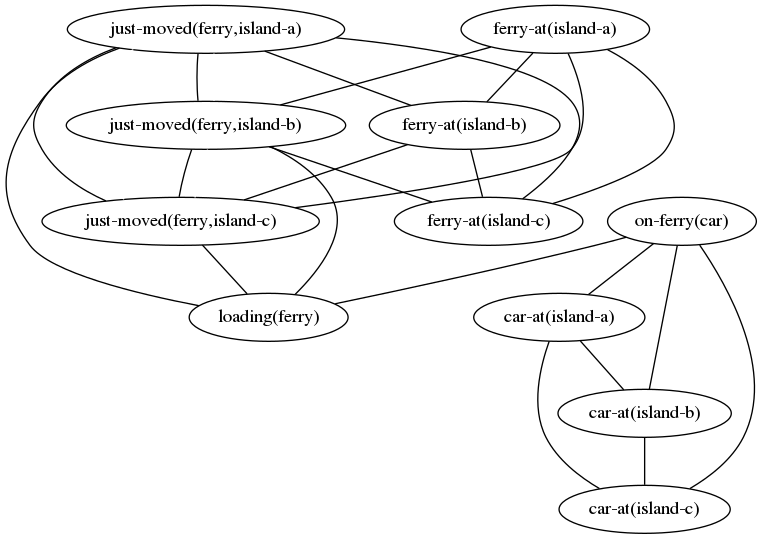}
\caption{Mutex graph for the ferry problem.}
\label{ferry1}
\end{figure}
 
\newpage
Now let's run our multiclique cover algorithm on it. We get:

\begin{verbatim}
% Multiclique 0 has all singleton parts
:- {holds(just_moved(ferry,island_a),T);
    holds(just_moved(ferry,island_b),T);
    holds(just_moved(ferry, island_c),T);
    holds(loading(ferry),T)
    } > 1; step(T).

% Multiclique 1 has all singleton parts
:- {holds(car_at(island_a),T);
    holds(car_at(island_b),T);
    holds(car_at(island_c),T);
    holds(on_ferry(car),T)
    } > 1; step(T).

% Multiclique 2 has three non-singleton partitions
partitionHolds(part(2,0),T) :- holds(ferry_at(island_a),T).
partitionHolds(part(2,0),T) :- holds(just_moved(ferry,island_a),T).
partitionHolds(part(2,1),T) :- holds(ferry_at(island_b),T).
partitionHolds(part(2,1),T) :- holds(just_moved(ferry,island_b),T).
partitionHolds(part(2,2),T) :- holds(ferry_at(island_c),T).
partitionHolds(part(2,2),T) :- holds(just_moved(ferry,island_c),T).
:- {partitionHolds(part(2,0),T);
    partitionHolds(part(2,1),T);
    partitionHolds(part(2,2),T)} > 1; step(T).

% Multiclique 3 has two singleton parts and so is just a normal
% mutex constraint.
:- holds(loading(ferry),T); holds(on_ferry(car),T).
\end{verbatim}

In total we have (per-layer) a grounded 10 rules with 25 literals. Had we used
the naive encoding it would have been 22 rules with 44 literals so we can see
this encoding is quite a bit more compact.

To give a better picture, in Figure \ref{colored}
we color each edge with the multiclique to which it
belongs. Note that three of the edges ended up in two distinct multicliques and
so are duplicated in the image: 

\begin{figure}[H]
\label{colored}
\includegraphics[scale=0.5]{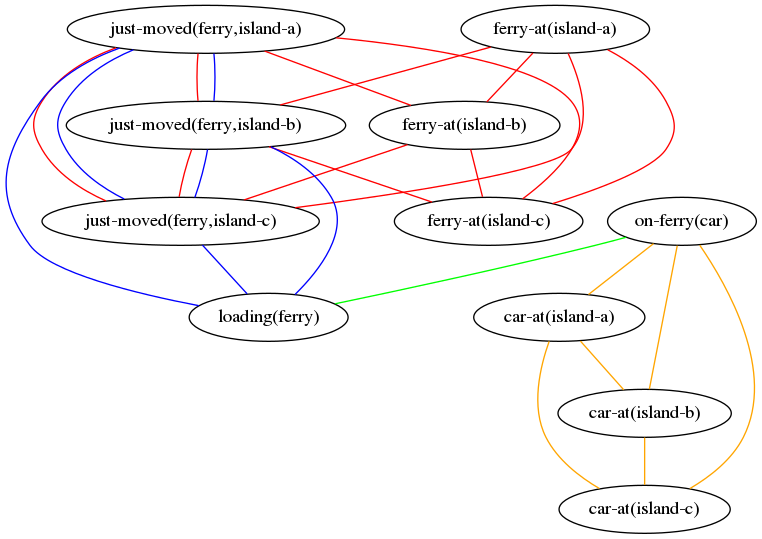}
\caption{Colored multiclique covering for the ferry problem.}
\end{figure}

\section{Eventual Fluent Mutex Constraints}
\label{action-mutex}
In Section \ref{ASPPlan} we found a way for the ASP solver
to avoid explicitly dealing with action mutex constraints and so were
able to save on grounded encoding space. But we still have a problem because
the algorithm presented by Blum and Furst \cite{blum1997fast} for \emph{generating} fluent
mutex constraints in the first place requires simultaneously constructing action
mutex constraints.

Indeed, Rintanen \cite{rintanen2006compact} reports being unable to run experiments
on the largest AIRPORTS instances from IPC-2004 because the action mutex
constraints used so much memory they wouldn't fit in a 32-bit address space.

In this section, we find a way to circumvent this problem and were able
to generate mutex constraints on the very largest (AIRPORTS-50) instance
while using only about a gigabyte of memory.

Mutex constraints as defined in \cite{blum1997fast} are ``per-layer''.
You determine the set of mutex constraints \emph{at each layer} by
looking at what actions, fluents and mutex constraints were in the previous
layer. Two actions are mutex if they are directly mutex or have any mutex
preconditions. Two fluents are mutex if all respective pairs of causing actions
are mutex. However, suppose we only care to discover and encode which fluents
are \emph{always} mutex in the sense that for \emph{every} layer up to an
arbitrarily large makespan they cannot both be true.

One way to obtain this set is to build the planning graph outward until the set
of mutex constraints stabilizes. That is, we can stop once we find two consecutive
layers at which the set of mutex constraints doesn't change. But this would still
require tracking action mutex constraints for all pairs of actions.

The key insight is that fluents which are always mutex will be so
in \emph{sequential} planning (where exactly one action happens at each layer) as well as in parallel planning.
A parallel plan is just a way of compressing a sequential plan into fewer steps
so the set of pairs of things which can be true at some point will be the same
regardless of how we express it.

Since a sequential plan can be expressed as a parallel plan where at most one
non-preserving action happens at each layer, we can run the mutex generation
algorithm under the assumption that \emph{all} non-preserving actions are mutex
with each other. Then we only need to explicitly keep track of which actions are
mutex with each of the \emph{preserving} actions. There are generally
significantly fewer preserving actions than total actions. When the set of
mutex fluent-pairs stabilizes, it should
come out the same as if we had obtained these pairs by building the planning
graph normally and waiting for the mutex fluents to stabilize.

\section{Experiments}
\label{experiment}

We implemented the multiclique generation algorithm in Haskell, representing
a fluent or action as an $Int$ and a collection of mutex constraints as an
$IntMap IntSet$. Both $IntMap$ and $IntSet$ come from the $containers$
package. A partition of a multiclique was represented as an $IntSet$,
a multiclique as a list of partitions, and a multiclique covering as a list of
multicliques.

We ran this algorithm on the same instances as Rintanen (as well as on the
AIRPORTS-50 instance, the largest problem in the set) and found
a significant improvement over his results. Note that these edge-counts do \emph{not} take into account neededness. That
is, they cover many fluents and actions which are irrelevant to the goal of the problem
and are guaranteed not to be explored by the solver.
When we accounted for neededness we found the graphs got much smaller (approximately
5-fold).
But we chose not to utilize this so that our results would be better
comparable to Rintanen's.

In Table \ref{tab:Multiclique-Reduction-for},
``Edges'' is the number of edges in the mutex graph for each instance.
``CL''  is the number of grounded clauses (rules) we used
to encode this graph. These
clauses are a mix of binary constraints and ``at most 1'' cardinality
constraints. Because not all the clauses are binary, we are compelled
to give the sum number of literals among all the constraints. This
is the ``Lit'' column.

During our experiments, after a look at a couple of example instances, it became immediately
clear to us that the majority of edges belong to the first few multicliques
found. After that the number of edges covered per clause drops off rapidly. Thus,
if we are willing to forget a small percentage of the edges, we can
reduce the number of clauses necessary to encode the graph much further.
For each instance, we reran the multiclique generation algorithm terminating it
as soon as it had covered 90\% of the total number of edges.\footnote{In our coding for the experiments, similar to edges, the uncovered edges are just represented by another {\em IntMapIntSet}. }
The resulting numbers of edges covered,
clauses, and literals required are given respectively by the columns
``Edges*'', ``CL*'', and ``Lit*''.
``R-Lit'' gives the number of literals required for Rintanen's biclique encoding.
It's twice the number of constraints he reports \cite{rintanen2006compact}
since all his constraints are binary clauses (having exactly two literals).

\begin{table}
\normalsize
\label{tab:Multiclique-Reduction-for}
\caption{Multiclique Reduction for AIRPORTS (Abbreviated AP)}
\vspace{.1in}
\centering 
\begin{tabular}{|c||c|c|c|c|c|c|c|}
\cline{1-8}
Instance & Edges & CL & {\bf Lit} & Edges* &  CL* & {\bf Lit*} & {\bf R-Lit}\\
\cline{1-8}
AP-21 & 181884 & 7531 & 16437 & 166229 & 2336 & 4783 & 26382 \\
AP-22 & 275515 & 11310 & 25014 & 249173 & 3464 & 7104 & 42776\\
AP-23 & 371062 & 14969 & 33100 & 336209 & 4806 & 9929 & 63552\\
AP-24 & 373188 & 15353 & 33894 & 337385 & 4907 & 10103 & 60814\\
AP-25 & 467653 & 18834 & 41821 & 421181 & 6208 & 12816 & 83438\\
AP-26 & 566948 & 22507 & 50252 & 511401 & 8025 & 16625 & 100494\\
AP-27 & 571298 & 22777 & 50801 & 514978 & 8155 & 16890 & 107442\\
AP-28 & 669336 & 26488 & 59201 & 602737 & 9941 & 20616 & 132120\\
AP-36 & 324835 & 9870 & 21502 & 297160 & 3084 & 6306 & 37744\\
AP-37 & 490408 & 14826 & 32921 & 442256 & 4266 & 8696 & 61362\\
AP-38 & 487033 & 14678 & 32793 & 438457 & 4263 & 8682 & 58928\\
AP-39 & 654787 & 20501 & 45166 & 598421 & 6352 & 12965 & 89294\\
AP-40 & 656469 & 20486 & 45150 & 599396 & 6351 & 12956 & 87744\\
AP-41 & 653096 & 20241 & 44709 & 588884 & 5846 & 11914 & 84628\\
AP-50 & 2613736 & 76180 & 171944 & 2353222 & 34538 & 71644 & -\\
\cline{1-8}
\end{tabular}
\end{table}

It is worth mentioning that our implementation of multiclique covering has been employed in a cost-optimal planner in ASP  \cite{spies2019}. That is, all the experiment results reported in \cite{spies2019} for that planner used this implementation for the representation of mutex constraints, where every 
plan produced by the planner was validated by the Strathclyde Planning
Group plan verifier VAL \cite{howey2004val}.

\section{Related Work}
\label{related} 
In \cite{rintanen2006compact}, an algorithm called \textsc{identify-biclique} is presented. Given a graph $G = (V,E)$, the algorithm starts with the trivial biclique $\emptyset$, $V$, and repeatedly adds nodes to the first part. Nodes from the second part are removed if there is no edge between them and the new node in the first part. The nodes are chosen to maximize the size reduction. The algorithm terminates when the size reduction is no longer possible. 

Our algorithm on multiclique covering is a natural extension of the \textsc{identify-biclique}
algorithm with some key differences.

\begin{itemize}
  \item We're generating multicliques rather than bicliques so there can be
    more than two partitions. In contrast with Rintanen's explicit construction of two partitions, which is possible and convenient because of the limit on two, we generate partitions for a multiclique based on a graph-theoretic property.

  \item As commented earlier, instead of \emph{removing} edges from the graph once they've been
    assigned to a multiclique, we keep a separate record of ``uncovered'' edges
    which still remain to be assigned. The multiple uses of the same covered edges  can sometimes further minimize the encoding in ASP.
  \item Instead of optimizing for $|\mbox{Clauses in naive encoding}| - |\mbox{Clauses in our encoding}|$,
    we're optimizing for $|\mbox{\emph{Literals} in naive encoding}| - |\mbox{\emph{Literals} in our encoding}|$.\footnote{When dealing with strictly binary clauses (as in Rintanen's case),
    these behave identically since latter metric is just the former multiplied
    by two.}
\end{itemize}

Although in this paper we have focused on reducing grounding size for ASP planning, our algorithm can be applied to other applications in other knowledge representation languages. We note that in some logic-based knowledge representation languages, such as SAT, encoding multicliques may not be as convenient and compact as can be done in ASP though.  This gives ASP a major advantage. 

The grounding bottleneck for constraints has been tackled
in the literature with different approaches, e.g., 
by Cuteri et al. \cite{DBLP:journals/tplp/CuteriDRS17}, where
mutex constraints can be added on demand. A comparison with our approach merits a further investigation.

\section{Conclusion and Final Remarks}
\label{conclusion}
Mutex constraints can significantly contribute to the overall grounding size of a planning problem. These constraints can be represented by a mutex graph where vertices are fluents and edges represent exclusiveness.
In this paper, we address this problem by proposing an algorithm for a  multiclique covering from a given mutex graph. As computing a minimum multiclique covering from a mutex graph is NP-hard, we propose an intuitive, approximation algorithm and show experimentally that it generates substantially smaller grounding size for mutex constraints in ASP than the previously known work in SAT. 

Like \cite{rintanen2006compact}, our approximation algorithm does not provide any approximation guarantees. A question of interest is whether such a guarantee can be formulated and proved.

The benchmark used in the experiments reported in this paper is limited to the AIRPORT problem. Experiments using more benchmarks are needed. Following Rintanen \cite{rintanen2006compact},  our goal is to seek smaller grounding sizes, under the assumption that smaller grounding sizes are better. This assumption may not hold true in general in all cases, and needs to be tested out by more experiments.

\bibliographystyle{eptcs}
\bibliography{david}

\begin{thebibliography}{10}
\providecommand{\bibitemdeclare}[2]{}
\providecommand{\surnamestart}{}
\providecommand{\surnameend}{}
\providecommand{\urlprefix}{Available at }
\providecommand{\url}[1]{\texttt{#1}}
\providecommand{\href}[2]{\texttt{#2}}
\providecommand{\urlalt}[2]{\href{#1}{#2}}
\providecommand{\doi}[1]{doi:\urlalt{http://dx.doi.org/#1}{#1}}
\providecommand{\bibinfo}[2]{#2}

\bibitemdeclare{inproceedings}{amilhastre1997computing}
\bibitem{amilhastre1997computing}
\bibinfo{author}{J{\'e}r{\^o}me \surnamestart Amilhastre\surnameend},
  \bibinfo{author}{Philippe \surnamestart Janssen\surnameend} \&
  \bibinfo{author}{Marie-Catherine \surnamestart Vilarem\surnameend}
  (\bibinfo{year}{1997}): \emph{\bibinfo{title}{Computing a minimum biclique
  cover is polynomial for bipartite domino-free graphs}}.
\newblock In: {\sl \bibinfo{booktitle}{Proc. 8th Annual ACM-SIAM Symposium on
  Discrete Algorithms}}, pp. \bibinfo{pages}{36--42}.

\bibitemdeclare{inproceedings}{BiereCCZ99}
\bibitem{BiereCCZ99}
\bibinfo{author}{Armin \surnamestart Biere\surnameend},
  \bibinfo{author}{Alessandro \surnamestart Cimatti\surnameend},
  \bibinfo{author}{Edmund~M. \surnamestart Clarke\surnameend} \&
  \bibinfo{author}{Yunshan \surnamestart Zhu\surnameend}
  (\bibinfo{year}{1999}): \emph{\bibinfo{title}{Symbolic model checking without
  BDDs}}.
\newblock In: {\sl \bibinfo{booktitle}{Proc. Tools and Algorithms for
  Construction and Analysis of Systems, 5th International Conference, {TACAS}
  '99}}, pp. \bibinfo{pages}{193--207}, \doi{10.1007/3-540-49059-0\_14}.

\bibitemdeclare{article}{blum1997fast}
\bibitem{blum1997fast}
\bibinfo{author}{Avrim~L \surnamestart Blum\surnameend} \&
  \bibinfo{author}{Merrick~L \surnamestart Furst\surnameend}
  (\bibinfo{year}{1997}): \emph{\bibinfo{title}{Fast planning through planning
  graph analysis}}.
\newblock {\sl \bibinfo{journal}{Artificial Intelligence}}
  \bibinfo{volume}{90}(\bibinfo{number}{1}), pp. \bibinfo{pages}{281--300},
  \doi{10.1016/S0004-3702(96)00047-1}.

\bibitemdeclare{misc}{ASP-standard}
\bibitem{ASP-standard}
\bibinfo{author}{Francesco \surnamestart Calimeri\surnameend},
  \bibinfo{author}{Wolfgang \surnamestart Faber\surnameend},
  \bibinfo{author}{Martin \surnamestart Gebser\surnameend},
  \bibinfo{author}{Giovambattista \surnamestart Ianni\surnameend},
  \bibinfo{author}{Roland \surnamestart Kaminski\surnameend},
  \bibinfo{author}{Thomas \surnamestart Krennwallner\surnameend},
  \bibinfo{author}{Nicola \surnamestart Leone\surnameend},
  \bibinfo{author}{Francesco \surnamestart Ricca\surnameend} \&
  \bibinfo{author}{Torsten \surnamestart Schaub\surnameend}
  (\bibinfo{year}{2015}): \emph{\bibinfo{title}{{ASP-Core-2} Input Language
  Format}}.
\newblock
  \bibinfo{howpublished}{\url{https://www.mat.unical.it/aspcomp2013/files/ASP-CORE-2.01c.pdf}}.
\newblock \bibinfo{note}{ASP Standardization Working Group}.

\bibitemdeclare{article}{DBLP:journals/tplp/CuteriDRS17}
\bibitem{DBLP:journals/tplp/CuteriDRS17}
\bibinfo{author}{Bernardo \surnamestart Cuteri\surnameend},
  \bibinfo{author}{Carmine \surnamestart Dodaro\surnameend},
  \bibinfo{author}{Francesco \surnamestart Ricca\surnameend} \&
  \bibinfo{author}{Peter \surnamestart Sch{\"{u}}ller\surnameend}
  (\bibinfo{year}{2017}): \emph{\bibinfo{title}{Constraints, lazy constraints,
  or propagators in {ASP} solving: An empirical analysis}}.
\newblock {\sl \bibinfo{journal}{{TPLP}}}
  \bibinfo{volume}{17}(\bibinfo{number}{5-6}), pp. \bibinfo{pages}{780--799},
  \doi{10.1017/S1471068417000254}.

\bibitemdeclare{book}{GareyJ79}
\bibitem{GareyJ79}
\bibinfo{author}{M.~R. \surnamestart Garey\surnameend} \&
  \bibinfo{author}{David~S. \surnamestart Johnson\surnameend}
  (\bibinfo{year}{1979}): \emph{\bibinfo{title}{Computers and Intractability:
  {A} Guide to the Theory of NP-Completeness}}.
\newblock \bibinfo{publisher}{W. H. Freeman}.

\bibitemdeclare{article}{Hochbaum98}
\bibitem{Hochbaum98}
\bibinfo{author}{Dorit~S. \surnamestart Hochbaum\surnameend}
  (\bibinfo{year}{1998}): \emph{\bibinfo{title}{Approximating clique and
  biclique problems}}.
\newblock {\sl \bibinfo{journal}{J. Algorithms}}
  \bibinfo{volume}{29}(\bibinfo{number}{1}), pp. \bibinfo{pages}{174--200},
  \doi{10.1006/jagm.1998.0964}.

\bibitemdeclare{inproceedings}{howey2004val}
\bibitem{howey2004val}
\bibinfo{author}{Richard \surnamestart Howey\surnameend},
  \bibinfo{author}{Derek \surnamestart Long\surnameend} \&
  \bibinfo{author}{Maria \surnamestart Fox\surnameend} (\bibinfo{year}{2004}):
  \emph{\bibinfo{title}{VAL: Automatic plan validation, continuous effects and
  mixed initiative planning using {PDDL}}}.
\newblock In: {\sl \bibinfo{booktitle}{Proc. 16th IEEE Int'l Conference on
  Tools with Artificial Intelligence}}, pp. \bibinfo{pages}{294--301},
  \doi{10.1109/ICTAI.2004.120}.

\bibitemdeclare{inproceedings}{DBLP:conf/coco/Karp72}
\bibitem{DBLP:conf/coco/Karp72}
\bibinfo{author}{Richard~M. \surnamestart Karp\surnameend}
  (\bibinfo{year}{1972}): \emph{\bibinfo{title}{Reducibility among
  combinatorial problems}}.
\newblock In: {\sl \bibinfo{booktitle}{Proc. Symposium on the Complexity of
  Computer Computations, March 20-22, 1972, at the {IBM} Thomas J. Watson
  Research Center, Yorktown Heights, New York, {USA}}}, pp.
  \bibinfo{pages}{85--103}.

\bibitemdeclare{article}{kautz2004satplan04}
\bibitem{kautz2004satplan04}
\bibinfo{author}{Henry \surnamestart Kautz\surnameend} (\bibinfo{year}{2004}):
  \emph{\bibinfo{title}{SATPLAN04: Planning as satisfiability}}.
\newblock {\sl \bibinfo{journal}{Working Notes on the Fourth Int'l Planning
  Competition (IPC-04)}}, pp. \bibinfo{pages}{44--45}.

\bibitemdeclare{book}{DBLP:books/daglib/0023091}
\bibitem{DBLP:books/daglib/0023091}
\bibinfo{author}{Daphne \surnamestart Koller\surnameend} \&
  \bibinfo{author}{Nir \surnamestart Friedman\surnameend}
  (\bibinfo{year}{2009}): \emph{\bibinfo{title}{Probabilistic Graphical Models
  - Principles and Techniques}}.
\newblock \bibinfo{publisher}{{MIT} Press}.

\bibitemdeclare{inproceedings}{cvpr/MaL12}
\bibitem{cvpr/MaL12}
\bibinfo{author}{Tianyang \surnamestart Ma\surnameend} \&
  \bibinfo{author}{Longin~Jan \surnamestart Latecki\surnameend}
  (\bibinfo{year}{2012}): \emph{\bibinfo{title}{Maximum weight cliques with
  mutex constraints for video object segmentation}}.
\newblock In: {\sl \bibinfo{booktitle}{Proc. 2012 {IEEE} Conference on Computer
  Vision and Pattern Recognition, Providence, RI, USA, June 16-21, 2012}}, pp.
  \bibinfo{pages}{670--677}, \doi{10.1109/CVPR.2012.6247735}.

\bibitemdeclare{article}{Peeters03}
\bibitem{Peeters03}
\bibinfo{author}{Ren{\'{e}} \surnamestart Peeters\surnameend}
  (\bibinfo{year}{2003}): \emph{\bibinfo{title}{The maximum edge biclique
  problem is NP-complete}}.
\newblock {\sl \bibinfo{journal}{Discrete Applied Mathematics}}
  \bibinfo{volume}{131}(\bibinfo{number}{3}), pp. \bibinfo{pages}{651--654},
  \doi{10.1016/S0166-218X(03)00333-0}.

\bibitemdeclare{inproceedings}{rintanen2006compact}
\bibitem{rintanen2006compact}
\bibinfo{author}{Jussi \surnamestart Rintanen\surnameend}
  (\bibinfo{year}{2006}): \emph{\bibinfo{title}{Compact representation of sets
  of binary constraints}}.
\newblock In: {\sl \bibinfo{booktitle}{Proc. 17th European Conference on
  Artificial Intelligence}}, pp. \bibinfo{pages}{143--147}.

\bibitemdeclare{inproceedings}{robinson2008compact}
\bibitem{robinson2008compact}
\bibinfo{author}{Nathan \surnamestart Robinson\surnameend},
  \bibinfo{author}{Charles \surnamestart Gretton\surnameend},
  \bibinfo{author}{Duc~Nghia \surnamestart Pham\surnameend} \&
  \bibinfo{author}{Abdul \surnamestart Sattar\surnameend}
  (\bibinfo{year}{2008}): \emph{\bibinfo{title}{A compact and efficient {SAT}
  encoding for planning.}}
\newblock In: {\sl \bibinfo{booktitle}{Proc. International Conference on
  Automated Planning and Scheduling}}, pp. \bibinfo{pages}{296--303}.

\bibitemdeclare{article}{robson1986algorithms}
\bibitem{robson1986algorithms}
\bibinfo{author}{John~Michael \surnamestart Robson\surnameend}
  (\bibinfo{year}{1986}): \emph{\bibinfo{title}{Algorithms for maximum
  independent sets}}.
\newblock {\sl \bibinfo{journal}{Journal of Algorithms}}
  \bibinfo{volume}{7}(\bibinfo{number}{3}), pp. \bibinfo{pages}{425--440},
  \doi{10.1016/0196-6774(86)90032-5}.

\bibitemdeclare{inproceedings}{cvpr/RoyT14}
\bibitem{cvpr/RoyT14}
\bibinfo{author}{Anirban \surnamestart Roy\surnameend} \&
  \bibinfo{author}{Sinisa \surnamestart Todorovic\surnameend}
  (\bibinfo{year}{2014}): \emph{\bibinfo{title}{Scene Labeling Using Beam
  Search under Mutex Constraints}}.
\newblock In: {\sl \bibinfo{booktitle}{2014 {IEEE} Conference on Computer
  Vision and Pattern Recognition, {CVPR} 2014, Columbus, OH, USA, June 23-28,
  2014}}, pp. \bibinfo{pages}{1178--1185}, \doi{10.1109/CVPR.2014.154}.

\bibitemdeclare{mastersthesis}{spies2019}
\bibitem{spies2019}
\bibinfo{author}{David \surnamestart Spies\surnameend} (\bibinfo{year}{2019}):
  \emph{\bibinfo{title}{Domain-Independent Cost-Optimal Planning in {ASP}}}.
\newblock Master's thesis, \bibinfo{school}{University of Alberta},
  \bibinfo{address}{Edmonton, Canada}.

\bibitemdeclare{inproceedings}{Yannakakis78}
\bibitem{Yannakakis78}
\bibinfo{author}{Mihalis \surnamestart Yannakakis\surnameend}
  (\bibinfo{year}{1978}): \emph{\bibinfo{title}{Node- and edge-deletion
  NP-complete problems}}.
\newblock In: {\sl \bibinfo{booktitle}{Proc. 10th Annual {ACM} Symposium on
  Theory of Computing, May 1-3, 1978, San Diego, California, {USA}}}, pp.
  \bibinfo{pages}{253--264}, \doi{10.1145/800133.804355}.

\end{thebibliography}
\end{document}